\documentclass[10pt,journal,twocolumn]{IEEEtran}
\usepackage{cite}
\usepackage{amsmath,amssymb,amsfonts}
\usepackage{graphicx}
\usepackage{textcomp}
\usepackage{amsthm}
\usepackage{xcolor}

\newtheorem{theorem}{Theorem}[section]

\newtheorem{remark}{Remark}
\usepackage{multicol, blindtext}
\usepackage{lipsum}
\usepackage{afterpage,lipsum,capt-of,graphicx}
\usepackage{caption}
\usepackage{cuted}
\usepackage{algorithmicx}
\usepackage{algpseudocode}
\usepackage{algorithm}
\algblockdefx{MRepeat}{EndRepeat}{\textbf{repeat}}{}
\algnotext{EndRepeat}
\enlargethispage{-2cm}
\usepackage{mathtools}
\usepackage{pbox}
\usepackage{mathtools}
\usepackage{subcaption}
\usepackage{refcount}
\usepackage{flushend}

\begin{document}

\title{ \Large \bf Multi-Robot Coordination Under Physical Limitations}

\author{Tohid Kargar Tasooji, Sakineh Khodadadi

\thanks{ Tohid Kargar Tasooji is with the Department of Aerospace Engineering,
Toronto Metropolitan University, Toronto, ON M5B 2K3, Canada (e-mail:
tohid.kargartasooji@torontomu.ca).  Sakineh Khodadadi is with the Department of Electrical and Computer
Engineering, University of Alberta, Edmonton,AB,  AB T6G 1H9, Canada (email: sakineh@ualberta.ca).}
}

\maketitle

\begin{abstract} \small\baselineskip=9pt
Multi-robot coordination is fundamental to various applications, including autonomous exploration, search and rescue, and cooperative transportation. This paper presents an optimal consensus framework for multi-robot systems (MRSs) that ensures efficient rendezvous while minimizing energy consumption and addressing actuator constraints.
A critical challenge in real-world deployments is actuator limitations, particularly wheel velocity saturation, which can significantly degrade control performance. To address this issue, we incorporate Pontryagin’s Minimum Principle (PMP) into the control design, facilitating constrained optimization while ensuring system stability and feasibility. The resulting optimal control policy effectively balances coordination efficiency and energy consumption, even in the presence of actuation constraints.
The proposed framework is validated through extensive numerical simulations and real-world experiments conducted using a team of Robotarium mobile robots. The experimental results confirm that our control strategies achieve reliable and efficient coordinated rendezvous while addressing real-world challenges such as communication delays, sensor noise, and packet loss.
\end{abstract}

\section{Introduction} \IEEEPARstart{M}{ulti-robot} systems (MRSs) have become a fundamental element in modern robotics research, enabling tasks such as environmental monitoring, search and rescue, and cooperative exploration \cite{1,2,3,4,23,24,25,26, 32, 33, 34, 35, 36, 37, 43, 45, 46}. In these applications, effective rendezvous and coordination among robots are crucial. The rendezvous problem generally involves designing distributed control protocols that allow each robot, using only locally available information or data from nearby teammates, to converge to a common state or formation \cite{5}. When convergence is achieved while optimizing specific performance metrics, the problem is referred to as optimal rendezvous.

A number of studies have explored optimal control strategies for robots with various dynamic models, addressing optimization objectives such as rapid convergence \cite{6}, minimal energy consumption \cite{7}, and robust performance against disturbances \cite{8}. For instance, the linear quadratic regulator (LQR) framework has been effectively employed to address optimal control challenges in multi-robot settings \cite{9,10,11}. While previous research has largely focused on first- and second-order dynamic models \cite{12,13,14,32,33}, recent investigations have extended these ideas to more complex systems. Examples include distributed LQR designs for leader-follower scenarios \cite{15,16} and data-driven solutions for systems with switching topologies \cite{17,18,19,20}.

Practical multi-robot applications introduce additional complexities, primarily due to the inherent physical limitations of each robot. A key challenge is input saturation, where actuators are constrained in the maximum forces or torques they can generate. If not properly addressed, input saturation can lead to instability or degraded system performance. While several studies have explored the rendezvous problem under input saturation for linear systems \cite{28,29,30,31}, few have integrated the optimization of energy performance into their control designs.

Motivated by the aforementioned challenges, this paper develops a novel distributed optimal rendezvous framework for multi-robot systems with higher-order dynamics and input saturation constraints. The main contributions are summarized as follows:

\begin{enumerate} \item We formulate the optimal rendezvous problem for general linear multi-robot systems and derive a closed-form solution based on algebraic Riccati equations (AREs) that decouples robot dynamics from network topology. \item We incorporate input saturation constraints into the rendezvous design using Pontryagin’s minimum principle (PMP), providing rigorous stability guarantees while addressing practical limitations such as wheel velocity and acceleration bounds. \item We validate the proposed distributed control strategy through real-world experiments on Robotrium robots, demonstrating robust performance under communication delays, bandwidth limitations, and packet loss. \end{enumerate}
The rest of this paper is organized as follows. In Section II, we describe the problem statement and the necessary preliminaries. In Section III, we present the proposed optimal distributed protocols. Section IV showcases a case study of the implementation on mobile robots. Finally, Section V concludes the paper and summarizes the results.

\section{{Preliminaries}}
\subsection{Notation}
In this work, we use standard mathematical notation. The symbol $\mathbb{R}$ denotes the set of real numbers, and $\mathbb{R}^{n \times m}$ indicates the space of matrices with $n$ rows and $m$ columns. Similarly, $\mathbb{R}^{n}$ represents the set of $n$-dimensional real vectors, and $\|x\|$ denotes the Euclidean norm of a vector $x\in\mathbb{R}^{n}$. For any matrix $X \in \mathbb{R}^{n \times m}$, its transpose is represented by $X^T \in \mathbb{R}^{m \times n}$, and $I_n$ stands for the $n \times n$ identity matrix.

Given two matrices $X \in \mathbb{R}^{m \times n}$ and $Y \in \mathbb{R}^{p \times q}$, their Kronecker product, denoted by $X \otimes Y$, forms a block matrix of size $pm \times qn$ defined as:
\[
X \otimes Y =
\begin{bmatrix}
x_{11}Y & \cdots & x_{1n}Y \\
\vdots & \ddots & \vdots \\
x_{m1}Y & \cdots & x_{mn}Y
\end{bmatrix}.
\]
For a square matrix $X \in \mathbb{R}^{n \times n}$, the notation $X > 0$ (or $X \geq 0$) is used to indicate that $X$ is positive definite (or positive semidefinite), meaning all its eigenvalues are strictly positive (or non-negative).

We also define a saturation function, $\mathrm{sat}_{(u_{i,\min},u_{i,\max})}(u_{i})$, which constrains a control input $u_i$ within the bounds $u_{i,\min} \le u_i \le u_{i,\max}$. It is formally given by:
\[
\mathrm{sat}_{(u_{i,\min},u_{i,\max})}(u_i)=
\begin{cases} 
u_{i,\max}, & u_i \geq u_{i,\max}, \\
u_i, & u_{i,\min} \leq u_i \leq u_{i,\max}, \\
u_{i,\min}, & u_i \leq u_{i,\min}.
\end{cases}
\]

\subsection{Graph Theoretic Concepts}
To model the interactions among $N$ agents, we utilize concepts from graph theory. The communication network is characterized by a graph $\mathcal{G}=(\mathcal{V},\mathcal{E}, \mathcal{A})$, where:
\begin{itemize}
    \item $\mathcal{V}=\{v_{1},v_{2},\ldots,v_{N}\}$ is the set of nodes, each corresponding to a distinct agent, uniquely identified by the indices $1, 2, \dots, N$.
    \item $\mathcal{E}$ is the set of directed edges; an edge $(v_i,v_j)\in \mathcal{E}$ signifies that information flows from agent $i$ to agent $j$.
    \item $\mathcal{A}=[a_{ij}] \in \mathbb{R}^{N \times N}$ is the adjacency matrix where $a_{ij}=1$ if $(v_i,v_j)\in \mathcal{E}$ and $a_{ij}=0$ otherwise, with the convention that $a_{ii}=0$ for all $i$.
\end{itemize}

Additionally, the network structure can be further described using the graph Laplacian $\mathcal{L}=[l_{ij}]$. For $i\neq j$, the off-diagonal elements are defined as $l_{ij}=-a_{ij}$, and the diagonal entries are given by $l_{ii}=\sum_{j\neq i} a_{ij}$, representing the degree of node $i$. This Laplacian matrix plays a critical role in analyzing the connectivity and consensus properties of the multi-agent system.

\subsection{Problem Formulation}
We consider a fleet of $N$ mobile robots, each operating in a two-dimensional environment. The state of robot $i$ is defined by its position in the $x$ and $y$ directions, denoted as $p_{ix}(t)$ and $p_{iy}(t)$. The dynamics of the $i$-th robot are governed by the following first-order system:
\begin{equation}\label{eq:robot_dynamics_2d}
\dot{x}_i(t) = A x_i(t) + B u_i(t), \quad i = 1, 2, \dots, N,
\end{equation}
where the state vector for robot $i$ is given by
\[
x_i(t) = \begin{bmatrix} p_{ix}(t) \\ p_{iy}(t) \end{bmatrix},
\]
representing the position of robot $i$ in the $x$ and $y$ directions. The system matrices are given by
\[
A = \begin{bmatrix} 0 & 0 \\ 0 & 0 \end{bmatrix}, \quad B = \begin{bmatrix} 1 & 0 \\ 0 & 1 \end{bmatrix}.
\]

In this work, the objective is to design a distributed control law that ensures the robots rendezvous, i.e., all robots achieve the same position in the two-dimensional plane as time progresses. Specifically, we want the system to satisfy
 \begin{equation}{l} 
\lim_{t \to \infty} \| x_i(t) - x_j(t) \| = 0, \quad \forall\, i,j = 1,2,\dots, N.
\end{equation}
This condition ensures that the robots converge to a common position in the $x$-$y$ plane.

To quantify the deviation from rendezvous, we define the position error between robots $i$ and $j$ as
 \begin{equation}{l} 
\varepsilon_{ij}(t) = x_i(t) - x_j(t).
\end{equation}
The global error vector for all robots can be written as
\[
\varepsilon(t) = \begin{bmatrix} \varepsilon_{1}(t) \\ \varepsilon_{2}(t) \\ \vdots \\ \varepsilon_{N}(t) \end{bmatrix}.
\]

The error dynamics for the system can be expressed as
\begin{equation}\label{eq:error_dynamics_2d}
\dot{\varepsilon}(t) = \big(I_N \otimes A\big) \varepsilon(t) + \big(\mathcal{L} \otimes B\big) U(t),
\end{equation}
where $\mathcal{L}$ is the Laplacian matrix representing the communication topology between the robots, and $U(t) = \begin{bmatrix} u_1(t) \\ u_2(t) \\ \vdots \\ u_N(t) \end{bmatrix}$ represents the control inputs for all robots.

The control law for robot $i$ is designed to be distributed, relying only on the relative position information from its neighbors. The control input for robot $i$ is given by
\begin{equation}\label{eq:control_law_2d}
u_i(t) = -K \sum_{j \in \mathcal{N}_i} \big( x_i(t) - x_j(t) \big),
\end{equation}
where $\mathcal{N}_i$ denotes the set of neighbors of robot $i$, and $K$ is a positive gain matrix.

To evaluate the performance of the system, we define the global performance index (cost function) as
\begin{equation}\label{eq:performance_index_2d}
J = \int_{0}^{\infty} \frac{1}{2} \left( \varepsilon(t)^T Q \varepsilon(t) + U(t)^T R U(t) \right) dt,
\end{equation}
where $Q = Q^T \ge 0$ and $R = R^T > 0$ are positive semi-definite and positive definite weighting matrices, respectively. The performance index $J$ reflects a trade-off between minimizing the tracking error (rendezvous error) and controlling the effort expended by the robots.

In summary, the problem is to design a distributed control strategy that ensures the robots rendezvous at a common position, while minimizing the energy expenditure, as specified by the performance index $J$.

\section{Main Results}
In this section, we develop a global optimal control protocol for multi-robot systems (MRSs) that ensures all robots rendezvous—that is, they converge to a common position—while minimizing an energy-based cost function. We consider two scenarios as follows:
\subsection{Optimal Rendezvous Control without Bounded Control Input}

\begin{theorem} \label{thm:rendezvous}
Consider a multi-robot system with the global error dynamics given by
\begin{equation} \label{eq:global_error_dynamics}
\dot{\varepsilon}(t) = (I_N \otimes A)\varepsilon(t) + (\mathcal{L} \otimes B)U(t),
\end{equation}
where $\varepsilon(t) \in \mathbb{R}^{mN}$ is the stacked error vector (with $m$ being the dimension of each robot's state), $A \in \mathbb{R}^{m \times m}$ and $B \in \mathbb{R}^{m \times r}$ are the system matrices, $\mathcal{L} \in \mathbb{R}^{N \times N}$ is the Laplacian matrix corresponding to the communication topology, and $U(t) \in \mathbb{R}^{rN}$ is the stacked control input vector. Assume that the weighting matrices $Q=Q^T\succeq 0$ and $R=R^T\succ 0$ are given. Then, the distributed control law
\begin{equation} \label{eq:optimal_control_law}
U^*(t) = - (\mathcal{L} \otimes K) \varepsilon(t),
\end{equation}
with
\begin{equation} \label{eq:control_gain}
K = R^{-1}B^T P,
\end{equation}
ensures that the robots achieve rendezvous, i.e., 
\[
\lim_{t\to\infty}\|\varepsilon(t)\| = 0,
\]
while minimizing the quadratic performance index
\begin{equation} \label{eq:cost_function}
J = \int_{0}^{\infty} \frac{1}{2}\left(\varepsilon(t)^T Q\varepsilon(t) + U(t)^T R U(t)\right) dt.
\end{equation}
Here, $P \succ 0$ is the unique positive definite solution of the algebraic Riccati equation (ARE)
\begin{equation} \label{eq:ARE}
PA + A^T P + Q - PBR^{-1}B^T P = 0.
\end{equation}
\end{theorem}

\begin{proof}
The proof is presented in two parts: first, we derive the necessary conditions for optimality using Pontryagin’s Minimum Principle, and then we show that the proposed control law guarantees both optimality and asymptotic stability.

\textit{(i) Necessity:} Define the Hamiltonian function for the optimal control problem as
\begin{equation} \label{eq:Hamiltonian}
H(\varepsilon, U, \lambda) = -\left(\varepsilon^T Q \varepsilon + U^T R U\right) + \lambda^T\left[(I_N \otimes A)\varepsilon + (\mathcal{L} \otimes B)U\right],
\end{equation}
where $\lambda \in \mathbb{R}^{mN}$ is the costate vector. The optimality condition requires that
\begin{equation} \label{eq:optimality_cond}
\frac{\partial H}{\partial U} = -RU + (\mathcal{L} \otimes B^T)\lambda = 0,
\end{equation}
which implies that
\begin{equation} \label{eq:U_optimal}
U^* = R^{-1} (\mathcal{L} \otimes B^T) \lambda.
\end{equation}

Next, we assume a linear relation between the costate and the state error:
\begin{equation} \label{eq:costate_relation}
\lambda = - (I_N \otimes P) \varepsilon,
\end{equation}
with $P\succ 0$. Substituting (\ref{eq:costate_relation}) into (\ref{eq:U_optimal}) yields
\begin{equation} \label{eq:U_star_final}
U^* = - (\mathcal{L} \otimes R^{-1}B^T P)\varepsilon = - (\mathcal{L} \otimes K)\varepsilon,
\end{equation}
where the control gain is defined as in (\ref{eq:control_gain}).

Differentiating the assumed costate relation (\ref{eq:costate_relation}) with respect to time, we obtain
\begin{equation} \label{eq:costate_diff}
\dot{\lambda} = - (I_N \otimes P) \dot{\varepsilon}.
\end{equation}
Substituting the error dynamics (\ref{eq:global_error_dynamics}) into (\ref{eq:costate_diff}), we have
\begin{equation} \label{eq:lambda_dot}
\dot{\lambda} = - (I_N \otimes P)\left[(I_N \otimes A)\varepsilon + (\mathcal{L} \otimes B)U^*\right].
\end{equation}
On the other hand, the costate dynamics provided by Pontryagin’s Minimum Principle are given by
\begin{equation} \label{eq:costate_dynamics}
\dot{\lambda} = -\frac{\partial H}{\partial \varepsilon} = (I_N \otimes A^T P + I_N \otimes Q)\varepsilon.
\end{equation}
Equating (\ref{eq:lambda_dot}) and (\ref{eq:costate_dynamics}) and simplifying, we obtain
\begin{equation} \label{eq:ARE_derivation}
I_N \otimes \left(PA + A^T P + Q - PBR^{-1}B^T P\right) = 0.
\end{equation}
Since $I_N$ is nonsingular, this condition reduces to the ARE (\ref{eq:ARE}), establishing the necessity of the gain $K = R^{-1}B^T P$.

\textit{(ii) Sufficiency:} To prove that the control law (\ref{eq:optimal_control_law}) guarantees both optimality and rendezvous, consider the Lyapunov function candidate
\begin{equation} \label{eq:Lyapunov}
V(\varepsilon) = \frac{1}{2}\varepsilon^T (I_N \otimes P)\varepsilon.
\end{equation}
Differentiating $V$ along the trajectories of the error dynamics (\ref{eq:global_error_dynamics}) under the control law (\ref{eq:optimal_control_law}) gives
\begin{equation} \label{eq:V_dot}
\dot{V} = \varepsilon^T (I_N \otimes P)\left[(I_N \otimes A)\varepsilon + (\mathcal{L} \otimes B)U^*\right].
\end{equation}
Substituting $U^* = - (\mathcal{L} \otimes K)\varepsilon$ and using the ARE (\ref{eq:ARE}), after straightforward algebra it follows that
\begin{equation} \label{eq:V_dot_final}
\dot{V} = -\frac{1}{2}\left[\varepsilon^T (I_N \otimes Q)\varepsilon + U^{*T}R U^*\right] \leq 0.
\end{equation}
Thus, $V$ is a valid Lyapunov function, and by LaSalle’s invariance principle, the error $\varepsilon(t)$ converges to zero as $t\to\infty$. Consequently, the robots achieve rendezvous.

Furthermore, the performance index (\ref{eq:cost_function}) can be expressed as
\begin{equation} \label{eq:J_final}
J = \int_0^{\infty} \frac{1}{2}\left[\varepsilon^T Q \varepsilon + U^{*T}R U^*\right]dt = V(0) - \lim_{t\to\infty}V(t) = V(0),
\end{equation}
which is minimized by the proposed control law.

This completes the proof.
\end{proof}

\subsection{Optimal Rendezvous Control with Bounded Control Input}

In this part, we extend our previous results to the practically important case where the control inputs are subject to hard bounds. In particular, we consider a multi-robot system in which each robot must rendezvous (i.e., converge to a common position) while its control input is constrained within pre-specified limits. The objective is to minimize the quadratic performance index
\begin{equation} \label{eq:cost_function_bounded}
J = \int_{0}^{\infty} \frac{1}{2}\Big(\varepsilon^T Q \varepsilon + U^T R U\Big) dt,
\end{equation}
subject to the global error dynamics
\begin{equation} \label{eq:error_dynamics_bounded}
\dot{\varepsilon}(t) = (I_{N}\otimes A)\varepsilon(t) + (I_{N}\otimes B)U(t),
\end{equation}
and the control constraints
\begin{equation} \label{eq:control_constraints}
U_{min} \leq U(t) \leq U_{max}.
\end{equation}
Here, $\varepsilon(t)\in\mathbb{R}^{mN}$ denotes the stacked error vector, $A\in\mathbb{R}^{m\times m}$ and $B\in\mathbb{R}^{m\times r}$ represent the dynamics of an individual robot, $Q=Q^T\succeq 0$ and $R=R^T\succ 0$ are weighting matrices, and $U(t)\in\mathbb{R}^{rN}$ is the control input vector.

\begin{theorem} \label{thm:bounded_rendezvous}
Consider the multi-robot system with dynamics given in (\ref{eq:error_dynamics_bounded}) and subject to the control constraints (\ref{eq:control_constraints}). Define the unconstrained optimal control as
\begin{equation} \label{eq:unconstrained_control}
U_{unc}(t) = - (\mathcal{L} \otimes R^{-1}B^T P)\varepsilon(t),
\end{equation}
where $P\succ 0$ is the unique solution to the algebraic Riccati equation
\begin{equation} \label{eq:ARE_bounded}
PA + A^T P + Q - PBR^{-1}B^T P = 0.
\end{equation}
Then, the optimal bounded control input that minimizes (\ref{eq:cost_function_bounded}) is given by the following bang-bang structure:
\begin{equation} \label{eq:optimal_control_bounded}
U^*(t)=
\begin{cases}
U_{min}, & \text{if } \quad U_{unc}(t) < U_{min}, \\[1mm]
U_{unc}(t), & \text{if } \quad U_{min} \le U_{unc}(t) \le U_{max}, \\[1mm]
U_{max}, & \text{if } \quad U_{unc}(t) > U_{max}.
\end{cases}
\end{equation}
Furthermore, the switching instants, denoted by $T_{s_1}$ and $T_{s_2}$, which separate the saturated and unsaturated control regimes, satisfy the following equalities:
\begin{equation} \label{eq:switching_Ts1}
\scalebox{0.82}{$
 \begin{array}{l} 
\exp\Big((I_{N}\otimes A)T_{s_1}\Big)\bar{\varepsilon}_0 +  \int_{0}^{T_{s_1}} \exp\Big((I_{N}\otimes A)(T_{s_1}-\tau)\Big)  (\mathcal{L} \otimes B)  U_{max}\, d\tau  \\  = \exp\Big(\Big[(I_{N}\otimes A) - (\mathcal{L} \otimes R^{-1}B^T P)\Big]T_{s_1}\Big)\varepsilon_0,
\end{array}
$}
\end{equation}
\begin{equation} \label{eq:switching_Ts2}
\scalebox{0.82}{$
\begin{array}{l}
\exp\Big((I_{N}\otimes A)T_{s_2}\Big)\underline{\varepsilon}_0 + \int_{0}^{T_{s_2}} \exp\Big((I_{N}\otimes A)(T_{s_2}-\tau)\Big) (\mathcal{L} \otimes B)U_{min}\, d\tau \\ = \exp\Big(\Big[(I_{N}\otimes A) - (\mathcal{L} \otimes R^{-1}B^T P)\Big]T_{s_2}\Big)\varepsilon_0.
\end{array}
$}
\end{equation}
Here, $\bar{\varepsilon}_0$ and $\underline{\varepsilon}_0$ denote the initial conditions corresponding to the trajectories under the extreme controls $U_{max}$ and $U_{min}$, respectively.
\end{theorem}

\begin{proof}
The proof is based on the application of Pontryagin’s Minimum Principle (PMP).

\textit{(i) Formulation of the Hamiltonian:} Define the Hamiltonian for the optimization problem as
\begin{equation} \label{eq:Hamiltonian_bounded}
H(\varepsilon, U, \lambda) = \varepsilon^T Q \varepsilon + U^T R U + \lambda^T\Big[(I_{N}\otimes A)\varepsilon + (I_{N}\otimes B)U\Big],
\end{equation}
where $\lambda\in\mathbb{R}^{mN}$ is the costate vector.

\textit{(ii) Necessary Optimality Conditions:} The state and costate dynamics are given by
\begin{align}
\dot{\varepsilon}^*(t) &= \frac{\partial H}{\partial \lambda} = (I_{N}\otimes A)\varepsilon^*(t) + (I_{N}\otimes B)U^*(t), \label{eq:state_dynamics_PMP} \\
\dot{\lambda}^*(t) &= -\frac{\partial H}{\partial \varepsilon} = -2Q\varepsilon^*(t) - (I_{N}\otimes A^T)\lambda^*(t). \label{eq:costate_dynamics_PMP}
\end{align}

In the unconstrained case, the minimization condition $\frac{\partial H}{\partial U}=0$ leads to
\begin{equation} \label{eq:U_condition}
-2RU^*(t) + (I_{N}\otimes B^T)\lambda^*(t) = 0,
\end{equation}
which implies
\begin{equation} \label{eq:unconstrained_relation}
U_{unc}(t)=\frac{1}{2}R^{-1}(I_{N}\otimes B^T)\lambda^*(t).
\end{equation}
A linear state-feedback ansatz is assumed for the costate:
\begin{equation} \label{eq:costate_ansatz}
\lambda^*(t) = -2(I_{N}\otimes P)\varepsilon^*(t),
\end{equation}
so that the unconstrained control law becomes
\begin{equation} \label{eq:unconstrained_feedback}
U_{unc}(t) = -(\mathcal{L} \otimes R^{-1}B^T P)\varepsilon^*(t),
\end{equation}
where the matrix $P\succ 0$ satisfies the algebraic Riccati equation (\ref{eq:ARE_bounded}).

\textit{(iii) Incorporation of Control Constraints:} Since the control input is bounded as in (\ref{eq:control_constraints}), the optimal control is obtained by projecting $U_{unc}(t)$ onto the admissible set. Hence, the optimal bounded control is given by
\begin{equation} \label{eq:control_projection}
U^*(t)= \Pi_{[U_{min},U_{max}]}\Big(U_{unc}(t)\Big),
\end{equation}
which is equivalent to the bang-bang structure in (\ref{eq:optimal_control_bounded}). In other words, if $U_{unc}(t)$ exceeds the upper bound $U_{max}$ (or falls below the lower bound $U_{min}$), the control saturates accordingly.

\textit{(iv) Switching Conditions:} To determine the switching instants between the saturated and unsaturated regimes, we analyze the state evolution under the extreme controls. When the control is saturated at $U_{max}$, the error dynamics reduce to
\begin{equation} \label{eq:state_Umax}
\dot{\varepsilon}(t)= (I_{N}\otimes A)\varepsilon(t) + (\mathcal{L} \otimes B)U_{max}.
\end{equation}
Its solution can be expressed as
\begin{equation} \label{eq:solution_Umax}
\scalebox{0.8}{$
 \begin{array}{l} 
\varepsilon(t)= \exp\Big((I_{N}\otimes A)t\Big)\bar{\varepsilon}_0 + \int_{0}^{t}\exp\Big((I_{N}\otimes A)(t-\tau)\Big)(\mathcal{L}\otimes B)U_{max}\,d\tau.
\end{array}
$}
\end{equation}
Similarly, when the control is saturated at $U_{min}$, we obtain
\begin{equation} \label{eq:solution_Umin}
\scalebox{0.8}{$
 \begin{array}{l} 
\varepsilon(t)= \exp\Big((I_{N}\otimes A)t\Big)\underline{\varepsilon}_0 + \int_{0}^{t}\exp\Big((I_{N}\otimes A)(t-\tau)\Big)(\mathcal{L}\otimes B)U_{min}\,d\tau.
\end{array}
$}
\end{equation}
In the unsaturated regime, where $U^*(t)=U_{unc}(t)$, the error dynamics are governed by
\begin{equation} \label{eq:state_unsaturated}
\dot{\varepsilon}(t)= \Big[(I_{N}\otimes A) - (\mathcal{L}\otimes R^{-1}B^T P)\Big]\varepsilon(t),
\end{equation}
with solution
\begin{equation} \label{eq:solution_unsaturated}
\varepsilon(t)= \exp\Big(\Big[(I_{N}\otimes A)-(\mathcal{L}\otimes R^{-1}B^T P)\Big]t\Big)\varepsilon_0.
\end{equation}
Matching the state trajectories at the switching instants yields the equalities (\ref{eq:switching_Ts1}) and (\ref{eq:switching_Ts2}).

This completes the proof.
\end{proof}
\begin{remark}
It is worth noting that due to the imposed bounded inputs, the control law in Theorem III.2 achieves a sub-optimal solution for the rendezvous problem in multi-robot systems (MRSs). The interaction topology among robots, the convexity of the performance index, and the specific input constraints jointly influence the optimization outcome. A key challenge in designing distributed optimal control strategies for MRSs is the careful selection of weighting matrices and costate variables to ensure that the optimization problem is tractable under a global interaction topology while satisfying the imposed input constraints.
\end{remark}

\begin{theorem} \label{thm:optimal_rendezvous_saturation}
Consider the global error dynamics given in (\ref{eq:global_error_dynamics}) subject to input saturation. Under the distributed control protocol (\ref{eq:control_constraints}), the multi-robot system achieves optimal rendezvous; that is, the state error converges to zero as $t\to\infty$.
\end{theorem}

\begin{proof}
We begin by proposing the following Lyapunov function candidate:
\begin{equation} \label{eq:lyapunov_candidate}
V(\varepsilon) = \sum_{i=1}^{N} \int_{0}^{\varepsilon_{i}} P_{i}\,\mathrm{sat}_{(U_{i,\min},U_{i,\max})}(K_{i}s)\, ds,
\end{equation}
where $P_i > 0$ and $K_i$ are appropriately chosen matrices for each robot. By the properties of the saturation function, it follows that $V(\varepsilon) \ge 0$ for all $\varepsilon$, with $V(\varepsilon)=0$ if and only if $\varepsilon = 0$.

Next, we compute the time derivative of $V(\varepsilon)$ along the trajectories of the global error dynamics:
\begin{align}
\dot{V}(\varepsilon) &= \sum_{i=1}^{N} \mathrm{sat}_{(U_{i,\min},U_{i,\max})}^T(K_{i}\varepsilon_{i})\, P_{i}\,\dot{\varepsilon}_{i} \nonumber \\
&= \mathrm{sat}_{(U_{\min},U_{\max})}^T(K\varepsilon)\, P\,\dot{\varepsilon}, \label{eq:V_dot1}
\end{align}
where the stacked matrices $P$ and $K$ are defined appropriately, and the saturation function is applied element-wise.

Substituting the global error dynamics
\begin{equation*}
\dot{\varepsilon} = (I_{N}\otimes A)\varepsilon + (\mathcal{L} \otimes B)\,\mathrm{sat}_{(U_{\min},U_{\max})}(K\varepsilon)
\end{equation*}
into (\ref{eq:V_dot1}), we obtain
\begin{align}
\dot{V}(\varepsilon) &= \mathrm{sat}_{(U_{\min},U_{\max})}^T(K\varepsilon)\, P\, \big[(I_{N}\otimes A)\varepsilon \nonumber \\
&\quad + (\mathcal{L} \otimes B)\,\mathrm{sat}_{(U_{\min},U_{\max})}(K\varepsilon)\big]. \label{eq:V_dot2}
\end{align}

By appropriately designing the control gain matrices and invoking the properties of the weighting matrices (as in the derivation of the unsaturated optimal control), it can be shown that
\begin{equation} \label{eq:V_dot_final}
\dot{V}(\varepsilon) \le -\mathrm{sat}_{(U_{\min},U_{\max})}^T(K\varepsilon)\, Q\, \mathrm{sat}_{(U_{\min},U_{\max})}(K\varepsilon) \le 0,
\end{equation}
where $Q>0$. Note that even though the term $P(I_{N}\otimes A)$ may have eigenvalues at zero (especially in high-order dynamics), the negative definiteness of the second term ensures that $\dot{V}(\varepsilon)$ is negative semi-definite.

By invoking LaSalle's invariance principle, we conclude that the error $\varepsilon(t)$ converges to the largest invariant set in which $\dot{V}(\varepsilon)=0$. This set is characterized by $\mathrm{sat}_{(U_{\min},U_{\max})}(K\varepsilon)=0$, which, by design, implies $\varepsilon=0$. Hence, the multi-robot system achieves rendezvous under the control protocol (30).
\end{proof}

\begin{remark}
For the special case of first-order dynamics and given the saturation bounds, the final rendezvous error $\varepsilon^*$ satisfies:
\begin{enumerate}
    \item If $U_{i,\min} < 0$ and $U_{i,\max} > 0$ for $i=1,2,\dots,N$, then $\varepsilon_{\min}(0) < \varepsilon^* < \varepsilon_{\max}(0)$.
    \item If $U_{i,\min} = 0$ and $U_{i,\max} > 0$ for $i=1,2,\dots,N$, then $\varepsilon^* = \varepsilon_{\max}(0)$.
    \item If $U_{i,\min} < 0$ and $U_{i,\max} = 0$ for $i=1,2,\dots,N$, then $\varepsilon^* = \varepsilon_{\min}(0)$.
    \item If $U_{i,\min} = U_{i,\max} = 0$ for $i=1,2,\dots,N$, then $\varepsilon^* = \varepsilon(0)$.
\end{enumerate}
These conditions illustrate the effect of the saturation bounds on the final rendezvous state.
\end{remark}

\begin{algorithm}[!t]
\caption{Optimal Rendezvous Protocol for Mobile Robots with First-Order Dynamics}
\label{alg:rendezvous_protocol}
\begin{algorithmic}[1]
\Require 
    \Statex Positions: $x_{i}(0) \in \mathbb{R}^{n_i}$ for $i = 1,\ldots, N$, 
    \Statex $\{ x_j(0) \}_{j \in \vartheta \setminus \{i\}}$, and tolerance $\epsilon \ll 0.05$
\Ensure 
    \Statex Speed commands $v_{i}(t)$, positions $x_{i}(t)$, and performance indices $J_i$
    
\While{$\| x_{i}(t) - x_{j}(t) \| > \epsilon,\ \forall\, i,j$}
    \State Solve the matrix Riccati equation:
        \[
            P A + A^T P + Q - P B R^{-1} B^T P = 0
        \]
    \State Compute the unconstrained control (velocity):
        \[
            v_{i}(t) = u_{i}(t) = -K \sum_{j=1}^{N} a_{ij}(t)\bigl(x_{i}(t)-x_{j}(t)\bigr)
        \]
    \State Update the robot positions using:
        \[
            \dot{x}_{i}(t) = -K \sum_{j=1}^{N} a_{ij}(t)\bigl(x_{i}(t)-x_{j}(t)\bigr)
        \]
    \If {$R^{-1}B\,\lambda_{i}^*(t) > u_{i,\text{max}}$}
        \State Determine the switching time $t_{s2}$ from:
            \[
            u_{i,\text{min}}\, t_{s2} + x_{i}(0) = \exp(-R^{-1}B^T P\, t_{s2})\,x_{i}(0)
            \]
        \State Set 
            \[
            x_{i}(t) = u_{i,\text{min}}\,t + x_{i}(0)
            \]
    \ElsIf {$R^{-1}B\,\lambda_{i}^*(t) < u_{i,\text{min}}$}
        \State Determine the switching time $t_{s1}$ from:
            \[
            u_{i,\text{max}}\, t_{s1} + x_{i}(0) = \exp(-R^{-1}B^T P\, t_{s1})\,x_{i}(0)
            \]
        \State Set 
            \[
            x_{i}(t) = u_{i,\text{max}}\,t + x_{i}(0)
            \]
    \ElsIf {$u_{i,\text{min}} < u_{i}(t) < u_{i,\text{max}}$}
        \State Set 
            \[
            x_{i}(t) = \exp(-R^{-1}B^T P\, t)\,x_{i}(0)
            \]
    \EndIf
    \State Compute the performance index:
        \[
            J_i = \int_{0}^{\infty} \frac{1}{2}\Bigl(\varepsilon_{i}^T Q \varepsilon_{i} + u_{i}^T R u_{i}\Bigr)dt
        \]
\EndWhile
\end{algorithmic}
\end{algorithm}

\section{\textbf{Case study}}
In this section, we implement and test the consensus algorithm 1 using (i) a simulated example and (ii) application to a team of e-puck2 mobile robots. Our objective is to design and implement optimal consensus protocol such that mobile robots achieve average consensus, {\it i.e.}, they converge to the agreement point and optimize the energy cost performance index. We analyze the convergence speed and energy cost for the proposed consensus algorithm by choosing different control gains. Note that, we consider the first-order dynamics for the consensus of mobile robots which is represented as follows:
\begin{equation} \label{eq:equation41}
       \begin{array}{l} 
             \boldsymbol{\dot x}_{i}(t)= v_{i}(t) \ \ \  i=1,2,...,N
            \end{array}
\end{equation}
where $\boldsymbol{ x}_{i}(t)$ and $v_{i}(t)$ are the position information and speed of mobile robots along x-axis, respectively. 
\subsection{Simulation Results}
\begin{figure}[h!] 
  \begin{center}
\includegraphics[width=0.3 \textwidth]{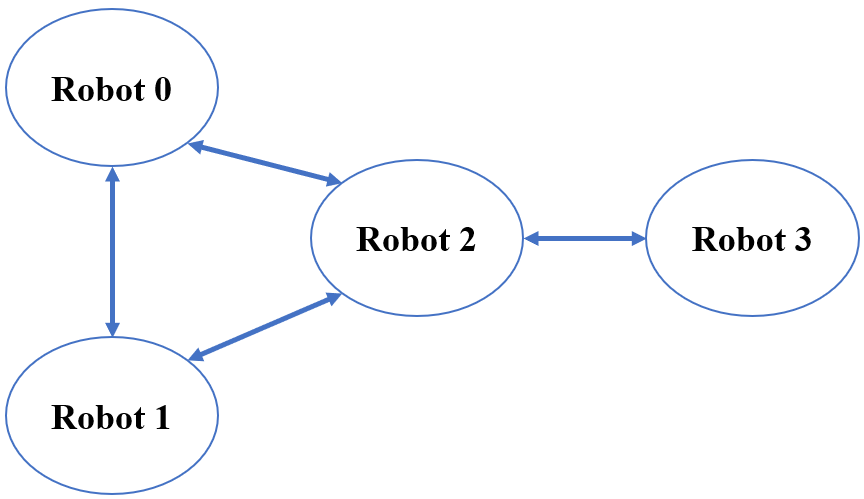}
  \end{center}
  \caption{Communication topology of the MAS} 
  \label{Fig3}
\end{figure} 

\begin{figure}[h!] 
  \begin{center}
\includegraphics[width=0.45 \textwidth]{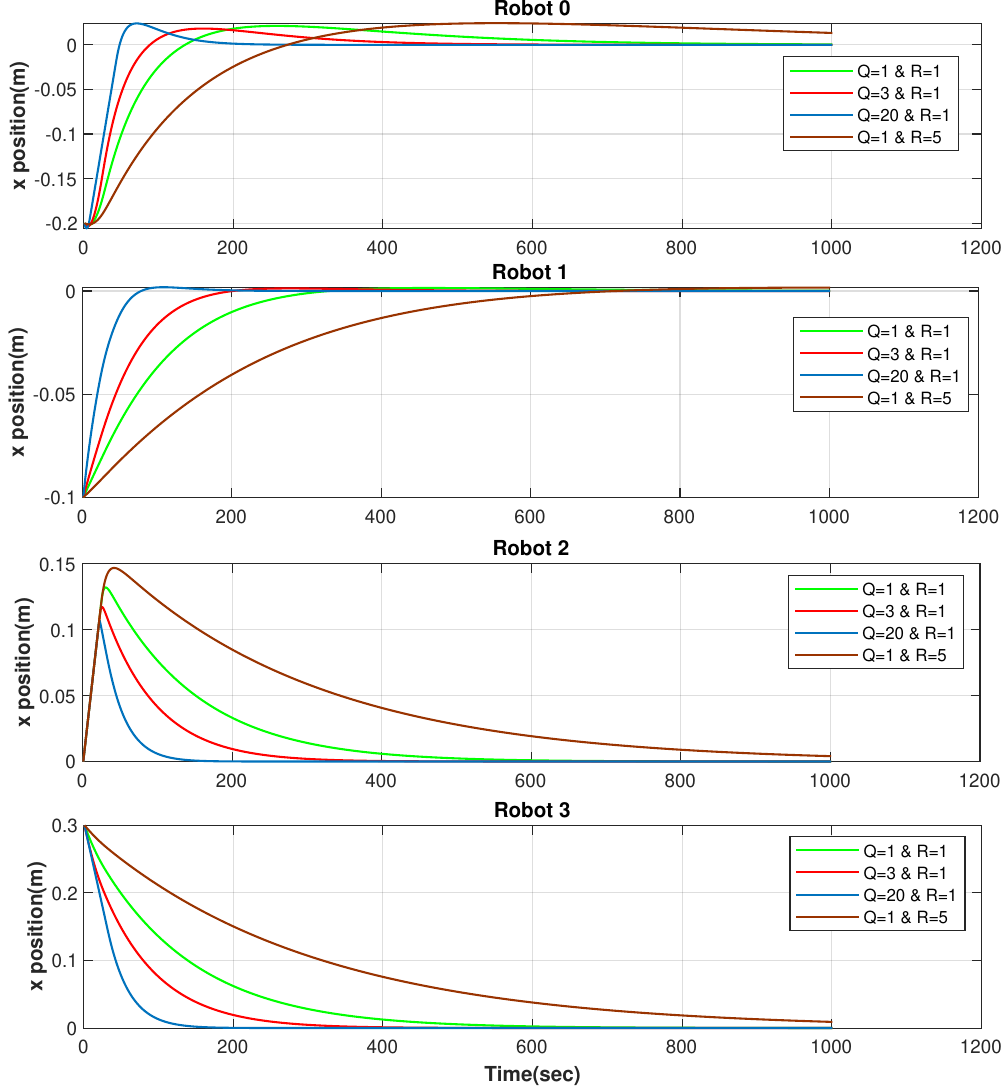}
  \end{center}
  \caption{Simulation results for optimal consensus algorithm 1: x positions of four mobile robots for different $Q$ and $R$.} 
  \label{Fig3}
\end{figure} 
\begin{figure}[h!] 
  \begin{center}
\includegraphics[width=0.45 \textwidth]{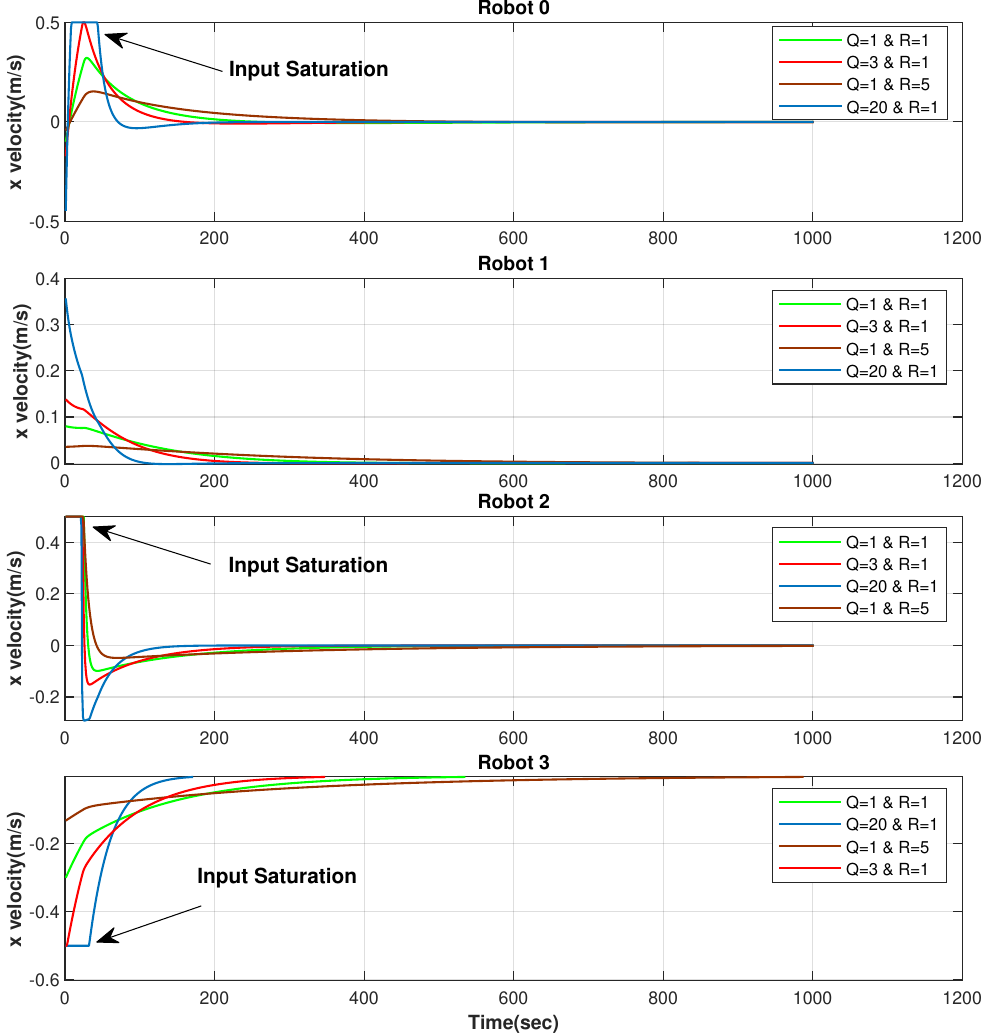}
  \end{center}
  \caption{Simulation results for optimal consensus algorithm 1: x velocities of four mobile robots for different $Q$ and $R$.}
  \label{Fig3}
\end{figure} 

\begin{figure}[h!] 
  \begin{center}
\includegraphics[width=0.45 \textwidth]{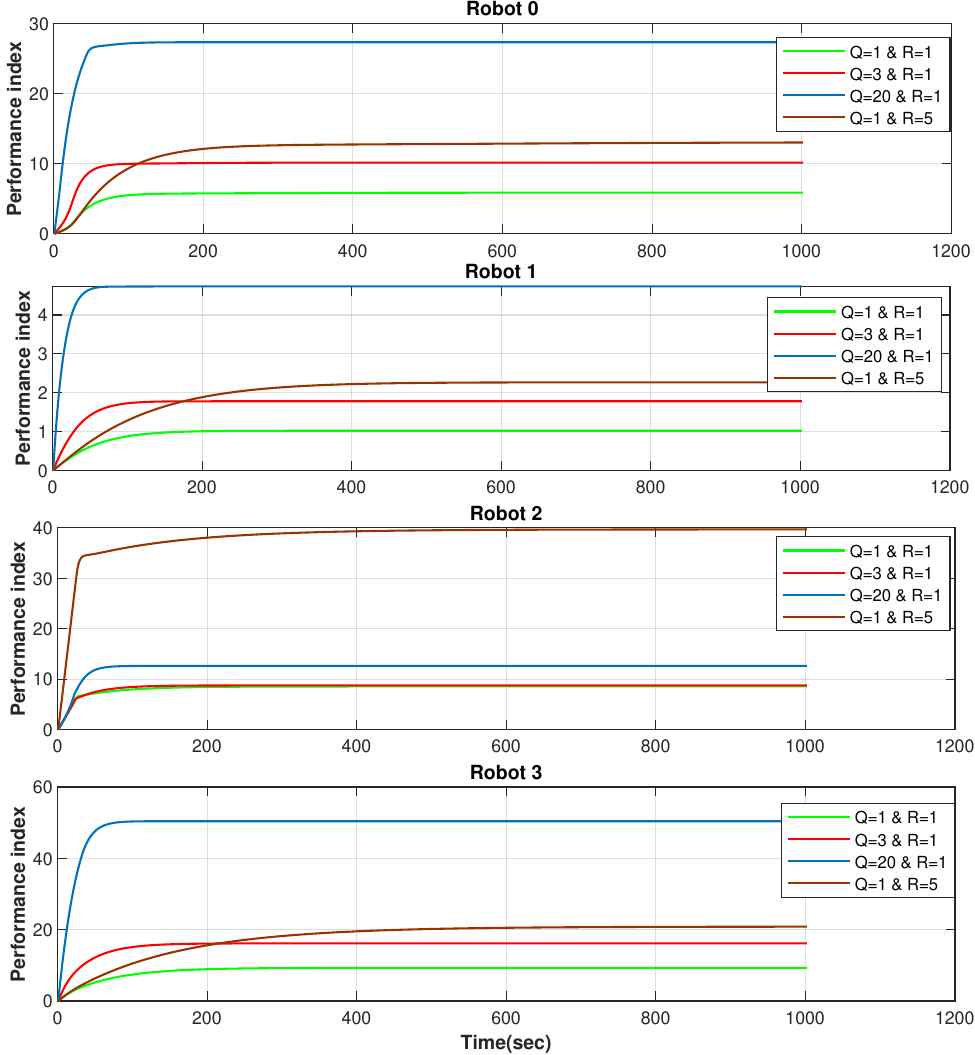}
  \end{center}
  \caption{Simulation results for optimal consensus algorithm 1: performance index (cost function) of four mobile robots for different $Q$ and $R$.}
  \label{Fig3}
\end{figure} 
In this section, we conduct couple of simulations under
MATLAB/Simulink environment to verify the effectiveness of
the designed optimal consensus control method for a group
of mobile robots.
Fig. 1 shows the communication topology of the MAS consisting of four mobile robots. The initial state of the robots are $x_{0}(0)=-0.2, x_{1}(0)=-0.1, x_{2}(0)=0$ and  $x_{3}(0)=0.3$. The matrix $P$ and control gain $K$ can be obtained using Theorem III.1  for any arbitrary choice of the matrices Q and R. Also, the mobile robot's wheel velocities are considered within the
bound of 0.5 m/s. Fig. 2, Fig. 3 and Fig. 4 show the simulation results of
the position, velocity and performance index for different matrices $Q$ and $R$. As expected, the results indicate that consensus is reached for all robots with different convergence speeds. In the case of $Q=3$ and $R=1$, the state cost term is penalized more compared to the input cost term, which means that the position response will have faster convergence and small transient response. However, the control effort magnitude will be larger. In the case of $Q=1$ and $R=5$, the input cost term is penalized compared to the state cost term. Therefore, the position response becomes more sluggish with larger transients and reduced control effort. When $Q=20$ and $R=1$, the velocity of the mobile robots is saturated with -0.5 $m/s$ and 0.5 $m/s$. In this case, the LQR method is no longer effective. Therefore, we consider Theorem III.2 for the optimal consensus of mobile robots where the input (velocity) is saturated. It can be seen that, whenever the speed of robots saturated, based on the result of Theorem  III.2 the optimal consensus control is maximum velocity of robots. Also, the motion of robots in this region will be linear. Once the control input satisfy within -0.5 $m/s$ and 0.5 $m/s$, the optimal consensus control will be obtained using  $ u_{i}(t)=-R^{-1}B^T P\sum_{j=0}^{N} { {a}}_{ij}({ {x}}_{i}(t)-{{x}}_{j}(t))$, which the motion of robots exponentially converge to the the agreement position. Also, the switching time is the time that the motion of robots switch from the linear motion to exponential motion.

\subsection{Experimental Validation}
\begin{figure}[h!] 
  \begin{center}
\includegraphics[width=0.45 \textwidth]{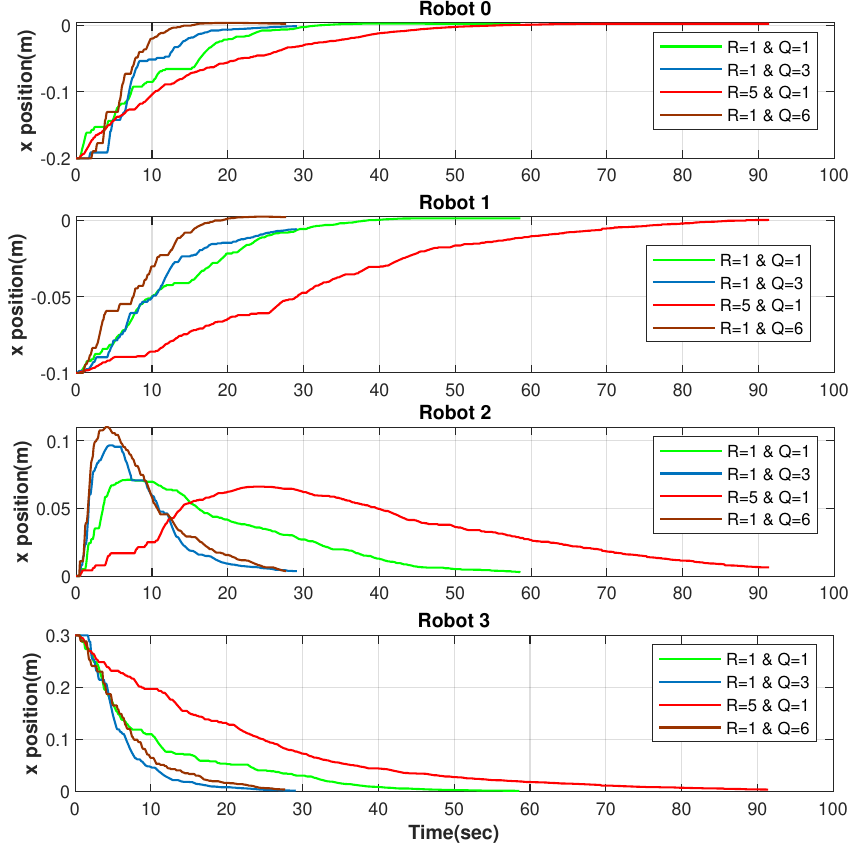}
  \end{center}
  \caption{Experimental testing results for optimal consensus algorithm 1: x positions of four mobile robots for different $Q$ and $R$.} 
  \label{Fig3}
\end{figure} 
\begin{figure}[h!] 
  \begin{center}
\includegraphics[width=0.45 \textwidth]{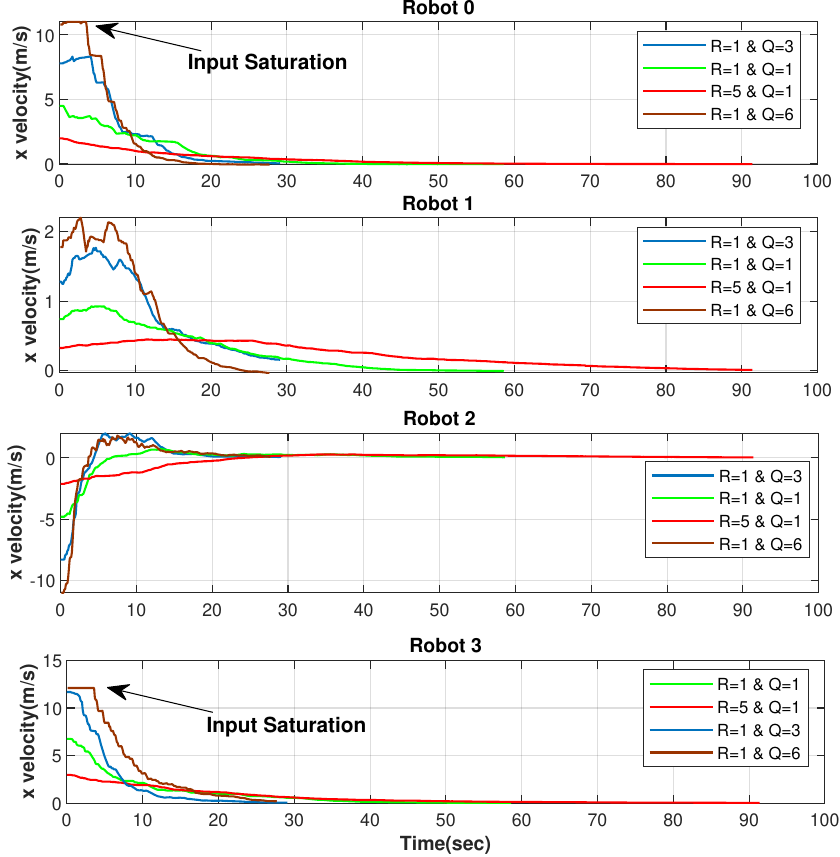}
  \end{center}
  \caption{Experimental testing results for optimal consensus algorithm 1: x velocities of four mobile robots for different $Q$ and $R$.} 
  \label{Fig3}
\end{figure} 

\begin{figure}[h!] 
  \begin{center}
\includegraphics[width=0.45 \textwidth]{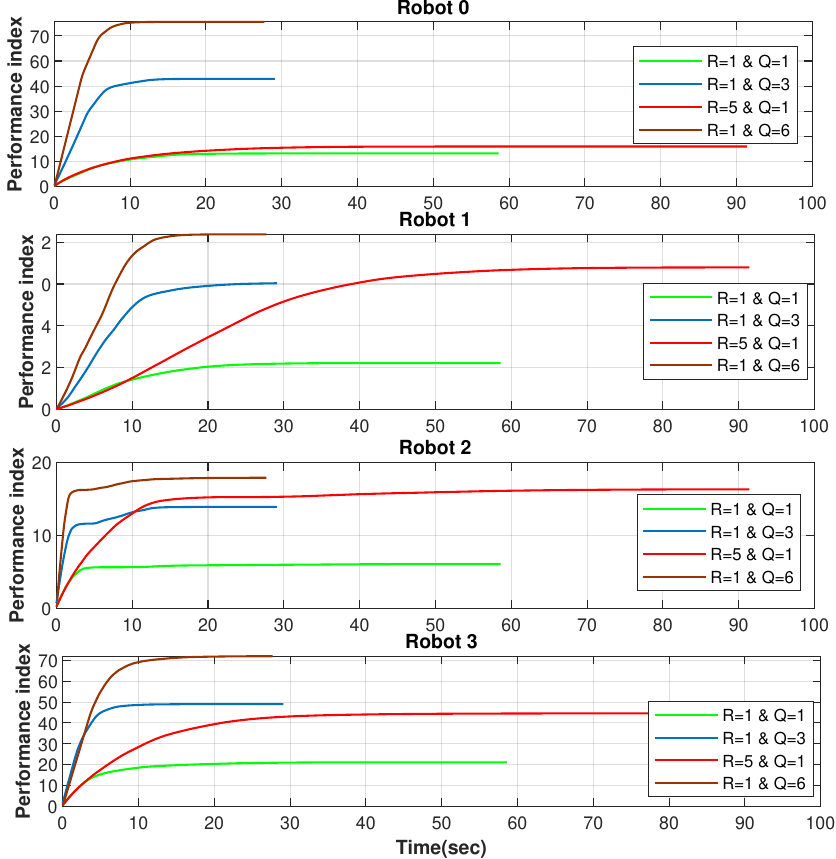}
  \end{center}
  \caption{Experimental testing results for optimal consensus algorithm 1: performance index (cost function) of four mobile robots for different $Q$ and $R$.} 
  \label{Fig3}
\end{figure} 

In this section, we validate the feasibility of the proposed optimal rendezvous algorithm by performing simulation experiments in the Robotarium platform. 

In our experiments, each robot relies on odometry measurements to navigate towards the desired position. Communication among the robots is achieved over a network, which introduces inherent time delays due to limited bandwidth. Initially, the robots are positioned at distinct locations (i.e., $x_{\text{robot0}}(0) = -0.2$, $x_{\text{robot1}}(0) = -0.1$, $x_{\text{robot2}}(0) = 0$, and $x_{\text{robot3}}(0) = 0.3$) with zero initial velocities. The left and right wheel speeds of the Robotrium robots are constrained such that $\| v_{l} \| \leq 11\,\mathrm{cm/s}$ and $\| v_{r} \| \leq 11\,\mathrm{cm/s}$.

Fig. 1 illustrates the communication topology of the multi-robot system, which plays a crucial role in determining the convergence speed of the system, as it is related to the smallest positive eigenvalue of the Laplacian matrix. Based on the communication topology, the third robot has a directed path to all other robots, and consequently, it converges to the desired position faster than its counterparts.

The optimal rendezvous control is applied to update the desired position of each robot at every time instant upon receiving updated position information over the network. The control updates are transmitted periodically to the rendezvous controller with a period of approximately 0.1 s.

 Figs. 5, 6, and 7 present the experimental results for position, velocity, and performance index under different choices of $Q$ and $R$. As expected, when the state (i.e., position) is penalized more heavily than the control effort (i.e., speed), a faster convergence is achieved at the cost of higher control inputs. Conversely, when the control effort is penalized more heavily, the position response exhibits slower transients. In the specific case of $Q=6$ and $R=1$, the velocities of the robots saturate at $-11\,\mathrm{cm/s}$ and $11\,\mathrm{cm/s}$, resulting in the application of the maximum or minimum allowable speed as the optimal control input. When operating within the linear region (i.e., below the saturation limits), the robot positions converge exponentially to the agreement position.

The experimental results confirm that our optimal rendezvous algorithm offers a robust framework for distributed multi-robot cooperative control, effectively handling physical limitations, wheel velocity constraints, packet loss, and time delays.

\section{Conclusion}
In this paper, we proposed an optimal rendezvous algorithm for multi-robot systems, aimed at achieving consensus in position while accounting for physical limitations, communication delays, and constraints on control inputs. Through extensive simulations and experimental validation using Robotrium platform, we demonstrated that the algorithm provides efficient convergence to a desired agreement position. The experimental results confirm that the proposed approach is robust under various conditions, such as limited communication bandwidth, time delays, and velocity constraints. Additionally, we analyzed the impact of different penalty parameters on convergence speed and control effort, highlighting the trade-off between faster convergence and higher control inputs. The results also show that the algorithm effectively handles robot physical limitations, ensuring stable and reliable multi-robot coordination in a distributed setup.

Future work will focus on extending the algorithm to handle more complex scenarios, including the presence of dynamic obstacles, variable network topologies, and more stringent time-varying constraints. Furthermore, we plan to investigate the scalability of the algorithm in larger multi-robot systems and explore its application in real-world robotic platforms for tasks such as autonomous transportation, multi-robot exploration, and cooperative manipulation.

\end{document}